\documentclass[12pt, english, oneside, a4paper]{article}

\usepackage[T1]{fontenc}
\usepackage{amsmath}
\usepackage{amssymb}
\usepackage{lscape}
\usepackage{amsthm}
\usepackage[left=2cm, right=2cm, top=2.5cm]{geometry}
\usepackage{hyperref}
\usepackage{algorithm}
\usepackage{bbm}
\usepackage{booktabs}
\usepackage{algpseudocode}
\usepackage{tikz}
\usepackage{caption}
\usepackage{subcaption}
\usepackage{fancyhdr}
\usepackage[inline]{enumitem}
\usepackage{comment}
\usepackage{nicefrac}
\usepackage{bm}
\usepackage{mathrsfs}
\usepackage{multirow}
\usepackage{graphicx}
\usepackage[utf8]{inputenc}
\usepackage{cancel}
\usepackage{mathtools}
\usepackage{natbib} 
\DeclareFontFamily{U}{rcjhbltx}{}
\DeclareFontShape{U}{rcjhbltx}{m}{n}{<->rcjhbltx}{}
\DeclareSymbolFont{hebrewletters}{U}{rcjhbltx}{m}{n}

\let\aleph\relax
\let\gimel\relax
\DeclareMathSymbol{\aleph}{\mathord}{hebrewletters}{39}
\DeclareMathSymbol{\gimel}{\mathord}{hebrewletters}{103}
\DeclareMathSymbol{\lamed}{\mathord}{hebrewletters}{108}

\newtheorem{theorem}{Theorem}
\newtheorem{lemma}[theorem]{Lemma}
\newtheorem{proposition}[theorem]{Proposition}
\newtheorem{corollary}{Corollary}[theorem]

\newtheorem{definition}[theorem]{Definition}
\newtheorem{remark}{Remark}

\newcommand{\R}{\mathbb{R}}

\renewcommand{\S}{\mathbb{S}}

\newcommand{\y}{\boldsymbol{y}}
\newcommand{\X}{\boldsymbol{X}}
\newcommand{\x}{\boldsymbol{x}}

\renewcommand{\a}{\boldsymbol{a}}
\newcommand{\bb}{\boldsymbol{\beta}}

\newcommand{\bt}{\boldsymbol{\theta}}
\newcommand{\bT}{\boldsymbol{\Theta}}

\title{Robust Predictive Uncertainty and Double Descent \\ in Contaminated Bayesian Random Features\thanks{To the best of our knowledge, this is the last manuscript on which SM worked actively before his untimely passing. We dedicate this work to his memory.}}

\author{Michele Caprio\thanks{The University of Manchester, UK}
\and Katerina Papagiannouli\thanks{University of Pisa, Italy}
  \and Siu Lun Chau\thanks{Nanyang Technological University, Singapore}
  \and Sayan Mukherjee\thanks{Max Planck Institute for Mathematics in the Sciences, Germany}}

\date{}

\begin{document}

\maketitle

\begin{abstract}

We propose a robust Bayesian formulation of random feature (RF) regression that accounts explicitly for prior and likelihood misspecification via Huber-style contamination sets. Starting from the classical equivalence between ridge-regularized RF training and Bayesian inference with Gaussian priors and likelihoods, we replace the single prior and likelihood with $\epsilon$- and $\eta$-contaminated credal sets, respectively, and perform inference using pessimistic generalized Bayesian updating. We derive explicit and tractable bounds for the resulting lower and upper posterior predictive densities. These bounds show that, when contamination is moderate, prior and likelihood ambiguity effectively acts as a direct contamination of the posterior predictive distribution, yielding uncertainty envelopes around the classical Gaussian predictive. We introduce an Imprecise Highest Density Region (IHDR) for robust predictive uncertainty quantification and show that it admits an efficient approximation via an adjusted Gaussian credible interval. We further obtain predictive variance bounds (under a mild truncation approximation for the upper bound) and prove that they preserve the leading-order proportional-growth asymptotics known for RF models. Together, these results establish a robustness theory for Bayesian random features: predictive uncertainty remains computationally tractable, inherits the classical double-descent phase structure, and is improved by explicit worst-case guarantees under bounded prior and likelihood misspecification.

\end{abstract}

\section{Introduction}\label{intro}

Random feature (RF) models are a classical and widely used mechanism to scale kernel methods to modern dataset sizes~\citep{rahimi2007random}. 
By replacing an implicit kernel with a finite collection of randomized nonlinear features, RF models turn nonlinear function learning into a linear problem in a randomized feature space, enabling fast training via least squares and ridge regression. 

Beyond their computational appeal, RF models have become a central object in the study of high-dimensional learning. In proportional asymptotic regimes where the number of samples $n$, ambient dimension $d$, and number of random features $N$ grow at comparable rates, prediction risk and predictive variance exhibit the now-classical double-descent phenomenon~\citep{belkin2019reconciling}. 
A precise asymptotic theory has since been developed for ridgeless least squares and random feature regression~\citep{mei2019generalization,hastie2022surprises}, showing that risk and variance display singular behavior near the interpolation threshold $N \approx n$, corresponding to $\psi_1 = N/d$ approaching $\psi_2 = n/d$. 
These works reveal that double descent is not merely empirical, but a structural feature of high-dimensional models.

A complementary and increasingly popular viewpoint interprets ridge-regularized RF training as Bayesian inference. In particular, the RF ridge solution can be understood as a maximum a posteriori (MAP) estimator. In standard Gaussian constructions, it coincides with the mean of a posterior predictive distribution for new responses \citep[Section~2.2]{baek2023asymptotics}. This Bayesian lens naturally supports uncertainty quantification, provides a principled route to hyperparameter selection, and connects RF models to broader probabilistic and statistical learning paradigms.

However, the classical Bayesian RF formulation relies on modeling choices that are rarely exact in practice. The Gaussian prior on the random-feature coefficients and the Gaussian likelihood for the outputs are convenient and analytically tractable, but they can be misspecified due to heavy-tailed noise, label corruption, heteroskedasticity, or distribution shift between training and testing. Even small deviations from these assumptions may yield posterior predictives that are overly confident, brittle to outliers, or sensitive to prior choices. This motivates the question we study:

\emph{How do Bayesian RF conclusions change when we explicitly account for prior and likelihood misspecification?}

\textbf{Our approach: a Contamination Random Features model.} We address this question through the lens of Bayesian sensitivity analysis \citep{berger} using Huber-style contamination sets \citep{walley,huber}. Concretely, instead of committing to a single prior and likelihood, we consider credal sets~\citep{levi2}, i.e. convex and weak$^\star$-closed sets of probabilities. Specifically, we work with an $\epsilon$-contaminated prior set and an $\eta$-contaminated likelihood set, obtained by mixing the baseline Gaussian specifications with arbitrary alternatives. This formalizes a principled ``ambiguity budget'' for both prior and data generating assumptions, and it also provides a way to model certain forms of distribution shift as long as the shift remains within the specified sets. This is reminiscent of the work in distributionally robust optimization with ambiguity sets   \citep{dellaporta2024distributionally,chen2026bulkcalibratedcredalambiguitysets}.

\paragraph{Choosing the contamination levels.}
The contamination levels \(\epsilon\) and \(\eta\) should be interpreted as robustness or sensitivity parameters. When domain information is available, \(\eta\) may be chosen as an upper bound on the expected fraction of corrupted, outlying, or distribution-shifted observations, while \(\epsilon\) encodes the analyst's tolerance to prior misspecification. In the absence of such information, we recommend reporting sensitivity curves over a grid of plausible values of \((\epsilon,\eta)\). When validation data are available, \(\eta\) can also be calibrated by choosing the smallest value for which the adjusted IHDR, or its Gaussian surrogate (see Section \ref{surrogate-sec}), attains the desired empirical predictive coverage on held-out data. Fully automatic selection of \((\epsilon,\eta)\), especially under arbitrary nonparametric contamination, requires additional modeling assumptions and is left for future work.

\begin{remark}[Role of Bayesian Sensitivity Analysis]
    The purpose of Bayesian sensitivity analysis is to assess the stability of inferential conclusions under departures from baseline modeling assumptions.
In our setting, given our ``ambiguity budgets'', we study how core predictive conclusions, such as predictive densities, uncertainty regions, and variance phase structure, change across this neighborhood.
This is a robustness objective---yielding conservative lower/upper posterior predictives---rather than an attempt to propose a new predictor that uniformly improves test error.
\end{remark}

A subtle but important technical point is that ensuring weak$^\star$ closedness---and hence the well-definedness of credal sets---may require working in the space of finitely additive probabilities \citep{walley}. While this generality can raise measure-theoretic complications, our analysis admits natural extremal contaminations (e.g., bounded spike densities concentrated near selected points), preserving an intuitive interpretation reminiscent of spike-and-slab behavior \citep{louzada}.

\textbf{Pessimistic generalized Bayes updating and predictive robustness.} Within the credal sets framework, a well-established method of carrying out posterior inference is via 
a set of posteriors, equivalently summarized by its lower and upper envelopes, called lower and upper posteriors, respectively. We adopt the update rule of a pessimistic generalized Bayesian agent \citep[Theorem~6.4.6]{walley}, which yields conservative (worst-case) posterior conclusions under the specified contamination sets. Our primary object of interest is predictive: the lower and upper posterior predictive distributions for a new response $\tilde{y}$ at a test input $\tilde{\x}$.

\textbf{Main results: tractable bounds for lower/upper posterior predictives.} Our first contribution is an explicit characterization of lower and upper densities associated with contamination credal sets, and their role in posterior updating. Building on this, we derive simple and interpretable bounds for the lower and upper posterior predictive densities induced by simultaneous prior and likelihood contamination. Let $\ell_{\mathrm{pred}}$ denote the classical RF posterior predictive density arising from the baseline Gaussian prior and likelihood (a one-dimensional normal distribution with mean given by the ridge RF predictor and variance given by the standard Bayesian RF expression \citep{baek2023asymptotics}). We prove that the \emph{lower} posterior predictive density $\underline{\ell}_{\mathrm{pred}}$ satisfies
\begin{equation*}
\underline{\ell}_{\mathrm{pred}} \le (1-\eta) \ell_{\mathrm{pred}},
\end{equation*}
and that the \emph{upper} posterior predictive density $\overline{\ell}_{\mathrm{pred}}$ satisfies a complementary inequality. 
These bounds highlight a simple phenomenon: likelihood contamination acts as an adversarial ``dilution'' of the baseline predictive density (and, in the upper case, an additive worst-case term), while prior contamination enters through the normalization and extremal updates. Importantly, when $\epsilon$ and $\eta$ are not too large, our derivations show that contaminating the prior and likelihood leads approximately to the same effect as directly contaminating the posterior predictive distribution, echoing recent observations in related ambiguity-set formulations \citep{dellaporta2024distributionally}.

\textbf{Uncertainty quantification via IHDRs.} A second contribution is a predictive uncertainty notion tailored to our imprecise predictive goal: the \emph{Imprecise Highest Density Region} (IHDR), defined as the smallest region that attains a target lower predictive probability mass. We show that, under the above predictive bounds, the IHDR can be approximated by a credible interval of the classical Gaussian posterior predictive distribution with an adjusted coverage level. This yields a practical, computable procedure for robust predictive intervals while preserving the analytic tractability of the RF Bayesian model.

\textbf{Moment consequences and asymptotics.} Finally, we connect these predictive density bounds to bounds on predictive dispersion. In particular, we upper bound the lower predictive variance by a simple rescaling of the classical predictive variance, and, under a mild truncation approximation, we obtain a corresponding lower bound for the upper predictive variance. Because the contamination parameters are constants, these variance bounds inherit the high-dimensional asymptotic behavior characterized in the RF literature \citep{baek2023asymptotics}. In other words, the Contamination RF model retains the same leading-order asymptotics as the baseline RF model, while providing a controlled robustness envelope around its predictions. 

\textbf{Contributions.}
In summary, our paper makes the following contributions:
\begin{itemize}
    \item We introduce the \emph{Contamination Random Features} model: a Bayesian RF formulation with $\epsilon$-contaminated priors and $\eta$-contaminated likelihoods represented as credal sets.
    \item We characterize lower and upper densities for contamination sets and derive tractable, explicit bounds on the lower and upper posterior predictive densities.
    \item We define predictive \emph{IHDRs} for imprecise posterior predictives and provide a practical approximation via adjusted Gaussian predictive credible intervals.
    \item We derive predictive variance bounds and show how these bounds align with known RF asymptotics in proportional-growth regimes.
\end{itemize}

\paragraph{Organization.}
The remainder of the paper is organized as follows. Section~\ref{sec:setup} introduces the RF model and its Bayesian interpretation. Section~\ref{sec:main} presents our main predictive bounds and their implications for IHDR construction. Sections~\ref{sec:variance} and \ref{doub-desc2} derive predictive variance bounds and discuss asymptotics. Finally, in Section~\ref{sec:experiments} we provide numerical experiments illustrating the predictive bounds and the preservation of double-descent structure under contamination. Section \ref{concl} concludes our work. Proofs and technical details are deferred to the appendix. 



\section{Random Feature (RF) Model}\label{sec:setup}
Let us briefly give a background on the classical Random Features Model. 
Denote inputs $\x_i\in\R^d,  i \in \{1,2,\ldots,n\}$, as i.i.d. samples drawn from a uniform measure on a $ d$-dimensional sphere with respect to the conventional Euclidean norm
\begin{equation*}
    \S^{d-1}(\sqrt{d}) := \{\x\in\R^{d}:\|\x\| = \sqrt{d}\}.
\end{equation*}
Let outputs $y_i$ be generated by the following model,
\begin{equation*}
    y_i = f_d(\x_i) + \gamma_i \text{,} \quad f_d(\x) = \beta_{d,0} + \langle\x,\bb_d\rangle + f_{d}^\text{NL}(\x).
\end{equation*}
The model is decomposed into a linear component, $\beta_{d,0} + \langle\x,\bb_d\rangle$, and a nonlinear component, $f_d^\text{NL}$. The random errors $\gamma_i$'s are assumed to be i.i.d. with mean zero, variance $\tau^2$, and finite fourth moment. We allow $\tau^2=0$, in which case the generating model is \emph{noiseless}.

We want to learn the optimal \emph{Random Features Model} that best fits the training data. This is a class of functions
\begin{equation}
    \mathcal{F} := \left\{f: f(\x) = \sum_{j=1}^{N} a_j\sigma(\langle\x,\bt_j\rangle/\sqrt{d})\right\},
    \label{eqn:rf_class}
\end{equation}
which is dependent on $N$ random features, $(\bt_j)_{j=1}^N$, which are drawn i.i.d. from $\mathbb{S}^{d-1}$,
and a nonlinear activation function, $\sigma$. The training objective solves a regularized least squares problem for the linear coefficients $\a\equiv(a_j)_{j=1}^{N}$,
\begin{align}\label{eqn:ridgep}
\begin{split}
    {\min}_{\a\in\R^N} \Bigg\{\sum_{i=1}^{n}&\left(y_i - \sum_{j=1}^{N} a_j\sigma\left(\langle \x_i,\theta_j\rangle/\sqrt{d}\right)\right)^2 \\ &+ d\psi_{1,d}\psi_{2,d}\lambda\|\a\|^2 \Bigg\},
\end{split}    
\end{align}
where $\psi_{1,d} = N/d$, $\psi_{2,d} = n/d$, and $\lambda > 0$. The optimal weights, $\widehat{\a}\equiv\widehat{\a}(\lambda)$, determine an optimal ridge predictor denoted by $\widehat{f}\equiv f(\cdot;\widehat{\a}(\lambda))$.\footnote{Given a new input feature $\tilde{\x}$, the ridge regularized predictor $\widehat{f}(\tilde{\x}) \equiv f(\tilde{\x};\widehat{\a})$ for the unknown output $\tilde{y}$ is given in \citet[Equation (5)]{baek2023asymptotics}. } The dependence of the trained predictor on the dataset $(\X,\y)$ and features $\bT$ are suppressed in notation.

The objective function \eqref{eqn:ridgep} of a classical Random Features (RF) model can be interpreted as a MAP estimation problem for an equivalent Bayesian model.\footnote{Here, a MAP has to be understood with respect to a posterior predictive distribution.} As we shall see more in detail later, the goal of this paper is to provide the Bayesian sensitivity analysis \citep{berger} version of the Bayesian RF model, via contamination sets \citep{huber}. We call it the {\em Contamination RF Model}.


Formally, we adopt a $d$-dependent weight prior distribution, denoted $p(\a)$, and also a normal likelihood, denoted $p(\y|\X,\bT,\a)$, centered around a function in the class \eqref{eqn:rf_class} with variance $\phi^{-1}$,
\begin{align}
    \a &\sim \mathcal{N}\left(0,\phi^{-1}\frac{\psi_{1,d}\psi_{2,d}\lambda}{d}\mathbf{I}_N\right), \label{normal_pr}\\
    \y \mid \X,\bT,\a &\sim \mathcal{N}\left(\sigma(\X\bT/\sqrt{d})\a,\phi^{-1}\mathbf{I}_n\right). \label{normal_lik}
\end{align}


\section{Bounding the Posterior Predictive Density}\label{sec:main}

We now turn to the central question of the paper: \emph{how does predictive inference in Bayesian random feature models change under explicit prior and likelihood misspecification?} To address this, we replace the single Gaussian prior and likelihood by contamination credal sets and analyze the resulting lower and upper posterior predictive.
To face such a possibility, we specify an $\epsilon$-contaminated prior credal set 
and an $\eta$-contaminated likelihood credal set. Formally, we let $P^\mathcal{N}$ denote the (countably additive) probability measure whose pdf is given by the Normal in \eqref{normal_pr}---which, for notational convenience, we denote by $p^\mathcal{N}$. Similarly, let $L^\mathcal{N} \equiv P(\cdot \mid \X,\bT,\a)$ denote the probability measure whose pdf is given by the Normal in \eqref{normal_lik}---which, for notational convenience, we denote by $\ell_{\a}^\mathcal{N} \equiv p(\cdot|\X,\bT,\a)$. In other words, given a sigma-finite dominating measure $\nu$ such as the Lebesgue one, we let $\text{d}P^\mathcal{N}/\text{d}\nu=p^\mathcal{N}$ and $\text{d}L^\mathcal{N}/\text{d}\nu=\ell_{\a}^\mathcal{N}$. Here, the ratios represent Radon-Nikodym derivatives. In the remainder of the paper, $\nu$ is left implicit for ease of exposition.

Pick any $\epsilon,\eta\in [0,1]$ and denote by $\Delta^\text{fa}_{\mathbb{R}^N}$ and $\Delta^\text{fa}_{\mathbb{R}^n}$ the space of finitely additive probability measures on $\mathbb{R}^N$ and $\mathbb{R}^n$, respectively. Then, the prior and likelihood contamination credal sets are the following,
\begin{align*}
    \mathcal{P}_\epsilon &\coloneqq \{P \in \Delta_{\mathbb{R}^N}^\text{fa} : P=(1-\epsilon)P^\mathcal{N} + \epsilon Q \text{, } \forall Q \in \Delta_{\mathbb{R}^N}^\text{fa}\},\\
    \mathcal{L}_\eta &\coloneqq \{L \in \Delta_{\mathbb{R}^n}^\text{fa} : L=(1-\eta)L^\mathcal{N} + \eta S \text{, } \forall S \in \Delta_{\mathbb{R}^n}^\text{fa}\}.
\end{align*}

\paragraph{Intuition for the contamination sets.}
For readers less familiar with imprecise probabilities, the main intuition is simple. The \(\epsilon\)-contaminated prior keeps a fraction \((1-\epsilon)\) of the baseline Gaussian prior \(P^{\mathcal N}\), while allowing the remaining \(\epsilon\) fraction to be placed adversarially. Therefore, for any proper measurable event \(A\), the smallest prior probability assigned to \(A\) is \((1-\epsilon)P^{\mathcal N}(A)\), while the largest is \((1-\epsilon)P^{\mathcal N}(A)+\epsilon\). The likelihood contamination set has the same interpretation, with \(\eta\) controlling the amount of likelihood misspecification. The finitely additive formulation below is used to make these credal sets mathematically closed and well-defined; the predictive bounds derived later depend only on the lower and upper envelope formulas.
 
 We need to consider finitely additive measures because, if we only considered countably additive ones, then the sets may be empty or not weak$^\star$ closed, and hence not well-defined credal sets \citep[Section 3]{caprio2025optimaltransportepsiloncontaminatedcredal}. It is easy to see that both $\mathcal{P}_\epsilon$ and $\mathcal{L}_\eta$ are convex \citep{marinacci2}. Recall that the lower envelope of a credal set is called {\em lower probability}. It is the smallest value that the elements of the credal set can take on; in our case, we have $\underline{P}\coloneqq \inf_{P\in\mathcal{P}_\epsilon} P$ and $\underline{L}\coloneqq \inf_{L\in\mathcal{L}_\eta} L$. The upper envelope is called {\em upper probability} (denoted by $\overline{P}$ and $\overline{L}$, respectively), and is defined analogously, with sup in place of inf. Notice that lower and upper probabilities are conjugate, that is, for all $A \in \mathcal{B}(\mathbb{R}^N)$, $\underline{P}(A)=1-\overline{P}(A^c)$, where $\mathcal{B}(\mathbb{R}^N)$ denotes the Borel sigma-algebra of $\mathbb{R}^N$, and similarly for the lower and upper likelihood. 
 The following is an important property of contamination sets, proved in \citet[Lemma 5]{caprio2025optimaltransportepsiloncontaminatedcredal}.\footnote{See also \citet{wasserman,walley}.}

\begin{lemma}[Characterizing the Contamination Sets]\label{char_lemma}
    The following are true,
    $$\underline{P}(A)=
    \begin{cases}
    (1-\epsilon)P^\mathcal{N}(A) & A \in \mathcal{B}(\mathbb{R}^N) \setminus \{\mathbb{R}^N\}\\
    1 & A=\mathbb{R}^N
    \end{cases} 
    $$
    and
$$\overline{P}(A)=
    \begin{cases}
    (1-\epsilon)P^\mathcal{N}(A) + \epsilon & A \in \mathcal{B}(\mathbb{R}^N) \setminus \{\emptyset\}\\
    0 & A=\emptyset
    \end{cases}.
    $$
    Similar characterization holds for $\overline{L}$ and $\underline{L}$. In addition,
    {\small \begin{align}\label{cores}
    \begin{split}
    \mathcal{P}_\epsilon &= \{P \in \Delta_{\mathbb{R}^N}^\text{fa} : P(A) \geq \underline{P}(A) \text{, } \forall A \in \mathcal{B}(\mathbb{R}^N)\} \eqqcolon \mathcal{M}(\underline{P}),\\
    \mathcal{L}_\eta &= \{L \in \Delta_{\mathbb{R}^n}^\text{fa} : L(B) \geq \underline{L}(B) \text{, } \forall B \in \mathcal{B}(\mathbb{R}^n)\} \eqqcolon \mathcal{M}(\underline{L}).
 \end{split}
\end{align}}
\end{lemma}

The sets $\mathcal{M}(\underline{P})$ and $\mathcal{M}(\underline{L})$ are called the {\em cores} of $\underline{P}$ and $\underline{L}$, respectively. They are defined as the collection of all (finitely additive) probabilities that set-wise dominate the lower probabilities. Thanks to Equation \eqref{cores}, we know that lower probabilities are sufficient to characterize the entire credal set---they are {\em compatible} with the credal sets \citep{gong}. 

To avoid unnecessary complications, we perpetrate abuse of notation and write $\underline{P}$ as the lower probability that assigns a value of $(1-\epsilon)P^\mathcal{N}(A)$ to all the elements $A$ of $\mathcal{B}(\mathbb{R}^N)$, and similarly for $\underline{L}$. We also write $\overline{P}$ as the upper probability that assigns a value of $(1-\epsilon)P^\mathcal{N}(A) + \epsilon$ to all the elements $A$ of $\mathcal{B}(\mathbb{R}^N)$, and similarly for $\overline{L}$. The interested reader can find more details around this choice in \citet[Section 3 and Appendix A]{caprio2025optimaltransportepsiloncontaminatedcredal}.

We now ask ourselves whether we are able to derive the upper and lower pdf's corresponding to the upper and lower probabilities, respectively. We find the answer to this question for $\underline{P}$ and $\overline{P}$, with the understanding that a similar argument holds also for $\underline{L}$ and $\overline{L}$.

\begin{lemma}[Lower and Upper Densities]\label{low_up_dens}
    The lower density $\underline{p}$ associated with lower probability $\underline{P}$ is such that, for all $A\in\mathcal{B}(\mathbb{R}^N)\setminus\{\mathbb{R}^N\}$, there exists a probability measure $Q_{\star A} \in \Delta_{\mathbb{R}^N}^\text{fa}$
    with $Q_{\star A}(A)=0$ and pdf $q_{\star A}$ such that
    $$\underline{P}(A)=\int_A [\underbrace{(1-\epsilon)p^\mathcal{N}(\a) + \epsilon q_{\star A}(\a)}_{\eqqcolon \underline{p}(\a)}] \text{d}\a.$$
    The upper density $\overline{p}$ associated with upper probability $\overline{P}$ is such that, for all $A\in\mathcal{B}(\mathbb{R}^N) \setminus \{\emptyset\}$, there exists a probability measure $Q_{A}^\star \in \Delta_{\mathbb{R}^N}^\text{fa}$ with $Q_{A}^\star(A)=1$ and pdf $q_{A}^\star$ such that 
    $$\overline{P}(A)=\int_A [\underbrace{(1-\epsilon)p^\mathcal{N}(\a) + \epsilon q_{A}^\star(\a)}_{\eqqcolon \overline{p}(\a)}] \text{d}\a.$$
\end{lemma}
The reader may view Lemma~\ref{low_up_dens} as a formal device for representing the lower and upper envelope probabilities at the density level; the subsequent predictive results use only envelope bounds and do not depend on choosing a unique extremal density.

Three remarks are in order. First, notice that the choice of $Q_{\star A}$ and $Q_{A}^\star$ depends on the set $A$ that we consider, and so do the definitions of $\underline{p}(\a)$ and $\overline{p}(\a)$. 
It is easy to see, then, that the global lower density envelope is $(1-\epsilon) p^\mathcal{N}$. In general, no finite $A$-independent global upper density envelope exists without further restrictions on the contaminating densities; for instance, if the contaminating densities are $M$-bounded a.e., for some $M>0$, then $(1-\epsilon) p^\mathcal{N} + \epsilon M$ is a valid pointwise majorant.

Second, since $Q_{\star A},Q_{A}^\star \in \Delta_{\mathbb{R}^N}^\text{fa}$, it may well be the case that they are finitely additive, but not countably additive. In that case, the concept of pdf is only well-defined in the sense of a Banach limit (which may not be unique), or as the pdf of a Loeb measure (which requires nonstandard analysis). Fortunately enough, though, in the problem that we consider, we can avoid such complications. Indeed, at the set-function level, it is convenient to think of the extremizers heuristically as Dirac masses.

Third, throughout, whenever we write lower/upper {\em densities}, we implicitly restrict attention to an absolutely continuous subclass of contaminations, 
so that all density-valued quantities and Lebesgue integrals appearing below are well-defined in the ordinary sense. In particular, whenever we write $\delta_{\mathbf y}$ (or $\delta_{\tilde y}$) inside expressions for lower/upper densities (e.g., $\underline{\ell}_{\a}(\mathbf y)$ or $\overline{\ell}_{\a}(\mathbf y)$), we mean a bounded spike density $s_{\mathbf y}$ with respect to Lebesgue measure, satisfying $s_{\mathbf y}\ge 0$, $\int s_{\mathbf y}(\mathbf z)\, d\mathbf z=1$, and $\|s_{\mathbf y}\|_{\infty}<\infty$ (for instance, the uniform density on a small hyper-rectangle centered at $\mathbf y$). We keep the notation $\delta_{\mathbf y}$ for readability. Similarly, whenever we write $\delta_{\a_\star}$ and $\delta_{\a^\star}$ for lower/upper prior densities, we mean bounded spike densities concentrated in a small neighborhood of $\a_\star$ and $\a^\star$, respectively.

\begin{remark}[On Envelope Densities]\label{rem:env-densities}
In what follows we derive (i) an \emph{upper bound} for the lower posterior predictive envelope density
$\underline{\ell}_{\mathrm{pred}}$ and (ii) a \emph{lower bound} for the upper posterior predictive envelope density
$\overline{\ell}_{\mathrm{pred}}$.
Consequently, we do \emph{not} require the existence of a finite, $A$-independent pointwise upper envelope (density majorant)
for the contaminating prior densities.
All bounds are obtained at the level of lower/upper probabilities 
and therefore do not
depend on such a pointwise majorant.
\end{remark}

Recall now that the optimal ridge weights $\hat{\a}$ that solve \eqref{eqn:ridgep} coincide with the posterior mean of the weight vector $\a$ under prior \eqref{normal_pr} and likelihood \eqref{normal_lik}. Consequently, given a new input feature $\tilde{\x}$, the posterior predictive distribution of $\tilde{y}$ induced by \eqref{normal_pr} and \eqref{normal_lik} has mean $\hat f(\tilde{\x})$. Hence, we are now interested in deriving the lower posterior predictive density.
Notice that we assume that the learner is a {\em pessimistic generalized Bayesian} agent \citep[Theorem 6.4.6]{walley}, that is, they update their beliefs in the form of an updated lower probability as
\begin{align}\label{pessim-gen-bayes}
    \underline{P}(\a \in A \mid \X,\bT,\y\in B)=\frac{\underline{P}(A) \underline{P}(B \mid \X,\bT,A)}{\overline{P}(B \mid \X,\bT)},
\end{align}
with $A \in \mathcal{B}(\mathbb{R}^N) \text{, } B \in \mathcal{B}(\mathbb{R}^n)$. For other types of updates, we refer the interested reader to \citet{walley,gong,teddy_me}; here, we limit ourselves to pointing out that pessimistic generalized Bayesian updating is a lower bound for both Walley’s classical generalized Bayesian updating \citep{walley} and the geometric updating rule \citep{suppes-zanotti}.\footnote{We will study its coherence \citep[Section 2.5]{walley} in future work.}

Let $\ell_\text{pred} \equiv p(\cdot \mid \tilde{\x},\y,\X,\bT)$ denote the density associated with the posterior predictive distribution ``induced'' by the Bayesian RF model encoded in \eqref{normal_pr} and \eqref{normal_lik}, and a new input $\tilde{\x}$. It is a one-dimensional Normal $\mathcal{N}(\hat{f}(\tilde{\x}),s^2(\tilde{\x}))$ whose parameters are given in \citet[Section 2.2]{baek2023asymptotics}.


\begin{theorem}[Bounding the Lower Posterior Predictive Density]\label{main_th}
    Let $\ell_\text{pred}$ denote the pdf of the one-dimensional posterior predictive Normal distribution 
    $\mathcal{N}(\hat{f}(\tilde{\x}),s^2(\tilde{\x}))$. Then, the following is true for all $\tilde y \in \mathbb R$,
    $$\underline{\ell}_\text{pred}(\tilde y) \leq (1-\eta) \ell_\text{pred} (\tilde y),$$
    where $\underline{\ell}_\text{pred}$ denotes the lower posterior predictive density.
\end{theorem}

In the following corollary, we bound the upper posterior predictive density. 





\begin{corollary}[Bounding the Upper Posterior Predictive Density]\label{cor-1}
    The following is true for all $\tilde{y}\in\mathbb{R}$,
    \[
    \overline{\ell}_\text{pred}(\tilde{y}) \ge (1-\eta) \ell_\text{pred}(\tilde{y}) + \eta u(\tilde{y}),
    \]
    where $\overline{\ell}_\text{pred}$ denotes the upper posterior predictive density and
    $u$ is any probability density on $\mathbb{R}$ (for convenience, one may take $u$ bounded).
\end{corollary}

As we can see from Theorem \ref{main_th} and Corollary \ref{cor-1}, if the prior and likelihood contaminations are not too extreme, contaminating prior and likelihood leads approximately to the same result as directly contaminating the posterior predictive distribution. 
This is conceptually related to the Bayesian ambiguity set perspective of \citet{dellaporta2024distributionally}, where posterior-informed ambiguity at the model level is used for downstream distributionally robust optimization. Our focus is different: rather than optimizing a worst case decision objective, we study robust posterior predictive inference for Bayesian random features, deriving explicit lower/upper predictive density bounds, IHDR approximations, and variance envelopes that preserve the double descent structure in proportional growth regimes.

\subsection{Approximating the Imprecise Highest Density Region}\label{surrogate-sec}

The lower posterior predictive density $\underline{\ell}_\text{pred}$ can be used to define the lower posterior predictive distribution as $\underline{P}(\tilde{y} \in B \mid \tilde{x},\y,\X,\bT)\coloneqq \int_B \underline{\ell}_\text{pred}(\tilde{y}) \text{d}\tilde{y}$, for all $B \in \mathcal{B}(\mathbb{R})\setminus \{\mathbb{R}\}$, and $\underline{P}(\tilde{y} \in \mathbb{R} \mid \tilde{x},\y,\X,\bT)=1$. Just as in the proof of Theorem \ref{main_th}, it can be routinely checked that $\underline{P}_\text{pred} \equiv \underline{P}(\cdot \mid \tilde{x},\y,\X,\bT)$ is a well-defined lower probability. In turn, we can use such a lower probability to derive a predictive Imprecise Highest Density Region \citep{coolen}.

\begin{definition}[Imprecise Highest Density Region---IHDR]\label{ihdr-def}
   Let $\alpha$ be any value in $[0,1]$. Then, the set $\text{IR}_\alpha[\mathcal{M}(\underline{P}_\text{pred})] \in \mathcal{B}(\mathbb{R})$ is called a $(1-\alpha)$-Imprecise Highest Density Region (IHDR) if 
    \begin{enumerate}
        \item \label{cond-1} $\underline{P}_\text{pred}\left(\{\tilde{y}\in \text{IR}_\alpha[\mathcal{M}(\underline{P}_\text{pred})]\}\right)= 1-\alpha$;
        \item \label{cond-2} $\int_{\text{IR}_\alpha[\mathcal{M}(\underline{P}_\text{pred})]} \text{d}\tilde{y}$ is a minimum. 
    \end{enumerate}
\end{definition}

Notice that Condition \ref{cond-2} is needed so that $\text{IR}_\alpha[\mathcal{M}(\underline{P}_\text{pred})]$ is the smallest possible subset of $\mathbb{R}$, which still satisfies Condition \ref{cond-1}. By the definition of lower probability, 
Definition \ref{ihdr-def} implies that $P_\text{pred}( \{\tilde{y}\in \text{IR}_\alpha[\mathcal{M}(\underline{P}_\text{pred})]\})\geq 1-\alpha$, for all ${P}\in\mathcal{M}(\underline{P}_\text{pred})$. Here lies the appeal of the IHDR concept. The following shows that it easy to approximate the IHDR.


\begin{proposition}[Approximating the IHDR]\label{char-ihdr}
    Pick any $\alpha \in [0,1]$.
    Let
    \(
    \beta\coloneqq \min\{1,\frac{1-\alpha}{1-\eta}\}.
    \)
    If $\text{IR}_\alpha[\mathcal{M}(\underline{P}_\text{pred})]$ exists, then
    \[
    \int_{R_\beta[P_\text{pred}]} \mathrm{d}\tilde{y}
    \le
    \int_{\text{IR}_\alpha[\mathcal{M}(\underline{P}_\text{pred})]} \mathrm{d}\tilde{y},
    \]
    where $R_\beta[P_\text{pred}]$ denotes the $\beta$-level credible interval of the posterior predictive distribution $\ell_\text{pred}=\mathcal{N}(\hat{f}(\tilde{\x}),s^2(\tilde{\x}))$.
    In particular, $R_\beta[P_\text{pred}]$ is a computable Gaussian surrogate for the IHDR.
\end{proposition}

\paragraph{Empirical choice of \(\eta\) for IHDR calibration.}
The adjusted IHDR approximation depends on the likelihood-contamination level \(\eta\). When labeled calibration data are available, \(\eta\) can be chosen empirically. For a target coverage level \(1-\alpha\), one may evaluate the Gaussian surrogate intervals
\(
C_{\eta}(x)
\coloneqq 
R_{\beta_{\eta}}\!\left[P_{\mathrm{pred}}(\cdot\mid x)\right] \), \(\beta \equiv \beta_{\eta}
\coloneqq 
\min\left\{1,\frac{1-\alpha}{1-\eta}\right\}
\), 
over a grid of values \(\eta\in\mathcal G\), and select the smallest \(\eta\) such that
\[
\frac{1}{m}\sum_{i=1}^{m}
\mathbf 1\{y_i^{\mathrm{cal}}\in C_{\eta}(x_i^{\mathrm{cal}})\}
\ge 1-\alpha,
\]
where $\mathbf 1\{\cdot \}$ denotes the indicator function. Among values satisfying this coverage constraint, one may further minimize average interval width or a proper interval score. Under arbitrary nonparametric contamination, \(\eta\) is not identifiable from unlabeled data alone without additional assumptions on the contamination mechanism. If no labeled calibration data or prior knowledge about misspecification are available, then \(\eta\) should be treated as a sensitivity parameter rather than as an identifiable quantity; in that case, the recommended practice is to report predictive envelopes over a range of plausible contamination levels.



\section{Bounding the Variance}\label{sec:variance}

In this section, we study bounds for the lower and upper predictive variances of the contaminated RF model. As a byproduct, we are able to state a result on the double descent behavior of our robust model.

Inspired by \citet[Appendix G.1]{walley}, we define the {\em lower predictive variance} as

\begin{align*}
\begin{split}    \mathbb{V}_{\underline{\ell}_\text{pred}}(\tilde{Y}) &\coloneqq \min_{\zeta\in\mathbb{R}} \mathbb{E}_{\underline{\ell}_\text{pred}}[(\tilde{Y}-\zeta)^2]\\
    &=\min_{\zeta\in\mathbb{R}} \int_{\mathbb{R}} [\tilde y - \zeta]^2 \underline{\ell}_\text{pred} (\tilde y) \text{d} \tilde y.
\end{split}
\end{align*}


Then, we have the following.

\begin{proposition}[Bounding the Lower Predictive Variance]\label{prop-var}
Let $\tilde Y$ denote the predictive quantity at $\tilde x$.
Under the classical posterior predictive, $\tilde Y$ has density
$\ell_\text{pred}=\mathcal N(\hat f(\tilde x),s^2(\tilde x))$,
whereas under the lower posterior predictive it has density $\underline{\ell}_\text{pred}$. Then, the following is true,
    $$\mathbb{V}_{\underline{\ell}_\text{pred}}(\tilde{Y}) \leq (1-\eta) \mathbb{V}_{{\ell}_\text{pred}}(\tilde{Y}).
    $$
\end{proposition}

Keeping the same notation, under an additional assumption, we can also bound the upper predictive variance, whose definition is once again inspired by  \citet[Appendix G.1]{walley},
\begin{equation*}
    \mathbb{V}_{\overline{\ell}_\text{pred}}(\tilde{Y}) \coloneqq \min_{\zeta\in\mathbb{R}} \mathbb{E}_{\overline{\ell}_\text{pred}}[(\tilde{Y}-\zeta)^2].
\end{equation*}


\begin{proposition}[Bounding the Upper Predictive Variance]\label{prop-var2}
Suppose we approximate the classical RF posterior predictive distribution
${\ell}_\text{pred}=\mathcal{N}(\hat{f}(\tilde{\x}),s^2(\tilde{\x}))$ with a truncated version
${\ell}_\text{pred}^\prime=\mathcal{N}_{[a,b]}(\hat{f}(\tilde{\x}),s^2(\tilde{\x}))$, where
$[a,b]\subset \mathbb{R}$, $a<b$. Let
\[
m \coloneqq \int_a^b \ell_\text{pred}(\tilde y)\,d\tilde y,
\qquad
u_{[a,b]}(\tilde{y}) \coloneqq \frac{1}{b-a}\mathbf{1}_{[a,b]}(\tilde{y}).
\]
Then,
\[
\mathbb{V}_{\overline{\ell}_\text{pred}}(\tilde{Y})
\ge
(1-\eta)m\, \mathbb{V}_{{\ell}_\text{pred}^\prime}(\tilde{Y}) + \eta \frac{(b-a)^2}{12}.
\]
In particular, if $[a,b]$ captures most of the Gaussian mass of $\ell_\text{pred}$, so that $m\approx 1$, then
\[
\mathbb{V}_{\overline{\ell}_\text{pred}}(\tilde{Y}) \gtrsim
(1-\eta) \mathbb{V}_{{\ell}_\text{pred}^\prime}(\tilde{Y}) + \eta \frac{(b-a)^2}{12}.
\]
\end{proposition}

We sum up our finding on the variance in the following.

\begin{lemma}[Predictive Variance Bounds]\label{var-compl}
    If ${\ell}_\text{pred} \approx {\ell}_\text{pred}^\prime$, $m \approx 1$, and $\frac{(b-a)^2}{12} \geq \mathbb{V}_{{\ell}_\text{pred}}(\tilde{Y})$, then it holds that
\begin{align}\label{var-bds}
\begin{split}
    \mathbb{V}_{\underline{\ell}_\text{pred}}(\tilde{Y}) &\leq (1-\eta) \mathbb{V}_{{\ell}_\text{pred}}(\tilde{Y}) \leq  \mathbb{V}_{{\ell}_\text{pred}}(\tilde{Y})\\
    &\leq 
    (1-\eta) \mathbb{V}_{{\ell}_\text{pred}}(\tilde{Y}) +  \eta \frac{(b-a)^2}{12}\\
    &\approx
    (1-\eta) \mathbb{V}_{{\ell}_\text{pred}^\prime}(\tilde{Y}) +  \eta \frac{(b-a)^2}{12}
     \lesssim \mathbb{V}_{\overline{\ell}_\text{pred}}(\tilde{Y}).
\end{split}
\end{align}
\end{lemma}

We note in passing that the conditions in Lemma \ref{var-compl} are always satisfied, because it is the user that selects $a$ and $b$, and so they can pick them to meet the requirements $m \approx 1$ and $\frac{(b-a)^2}{12} \geq \mathbb{V}_{{\ell}_\text{pred}}(\tilde{Y})$. For example, this holds if we use the classic heuristic of $a=\hat{f}(\tilde{\mathbf{x}}) - 3 s(\tilde{\mathbf{x}})$ and $b=\hat{f}(\tilde{\mathbf{x}}) + 3 s(\tilde{\mathbf{x}})$.

In addition, we use the truncation in Proposition \ref{prop-var2} because, without it, the function $\tilde{y} \mapsto \varphi(\tilde{y})\coloneqq(\tilde{y}-\zeta)^2$ need not be (Choquet) integrable with respect to the upper probability associated with the bound for the upper predictive density found in Corollary \ref{cor-1} \citep[Appendix C]{decooman}.

\begin{remark}[On the Choice of the Truncation Interval]
\label{rem:truncation-choice}
The bound on the upper predictive variance in Proposition~\ref{prop-var2}
depends on the truncation interval \([a,b]\), through the terms \(m\) and
\((b-a)^2/12\). If the interval is too narrow, the truncated predictive
density \(\ell^\prime_{\mathrm{pred}}\) deviates substantially from the
original Gaussian predictive density \(\ell_{\mathrm{pred}}\). If the
interval is too wide, the term \((b-a)^2/12\) becomes unnecessarily large,
leading to overly conservative uncertainty estimates.

On the one hand, to ensure that the truncated predictive density
\(\ell^\prime_{\mathrm{pred}}\) remains close to the Gaussian predictive
density
\[
\ell_{\mathrm{pred}}
=
\mathcal{N}\!\left(\hat f(\tilde{\mathbf{x}}),s^2(\tilde{\mathbf{x}})\right),
\]
the interval \([a,b]\) should contain most of the Gaussian mass. For the
symmetric choice
\[
a = \hat f(\tilde{\mathbf{x}}) - k s(\tilde{\mathbf{x}}),
\qquad
b = \hat f(\tilde{\mathbf{x}}) + k s(\tilde{\mathbf{x}}),
\]
the retained mass is
\(
m_k
=
\int_a^b \ell_{\mathrm{pred}}(\tilde y)\,d\tilde y
=
2\Phi(k)-1,
\)
where \(\Phi\) denotes the standard normal distribution function. Hence the
truncation error is explicitly
\(
1-m_k = 2\{1-\Phi(k)\}.
\)
For example, \(k=2\) gives \(m_k \approx 0.9545\), while \(k=3\) gives
\(m_k \approx 0.9973\). On the other hand, increasing \(k\) enlarges
\[
\frac{(b-a)^2}{12}
=
\frac{k^2}{3}s^2(\tilde{\mathbf{x}}),
\]
thereby enlarging the uniform-contamination contribution to the upper
variance envelope.

In the proportional asymptotic regime, a natural scaling choice is
\(k=O(1)\), independent of dimension, ensuring that the truncation width
remains proportional to the predictive standard deviation
\(s(\tilde{\mathbf{x}})\). Determining an optimal or data-adaptive choice
of \([a,b]\)---for instance, via calibration criteria or
decision-theoretic objectives---is an interesting direction for future work.
\end{remark}

\section{Variance and the Double Descent Phenomenon}\label{doub-desc2}
In this section, we work in the proportional growth random features regime, where
\[
d,n,N \to \infty
\quad\text{with}\quad
\psi_1 \coloneqq \nicefrac{N}{d}, \quad    \psi_2 \coloneqq \nicefrac{n}{d}
\]
fixed, and consider ridge-regularized random features regression with penalty $\lambda>0$.

\paragraph{Proof roadmap and interpretation.}
The double-descent preservation result below should be understood as the final step of the robustness analysis developed in Sections~\ref{sec:main} and \ref{sec:variance}. The technically substantive part is the derivation of the lower and upper posterior predictive density bounds and their translation into the variance envelope of Lemma~\ref{var-compl}. Once this envelope has been obtained, the preservation of the interpolation peak follows from the fact that the contamination bounds act as positive rescalings, or under the truncation choice below proportional transformations, of the baseline predictive variance. Thus, the result does not claim that affine invariance is technically difficult; rather, it shows that the robust contamination model reduces to a form that preserves the known double-descent phase structure of random-features regression.

Let $\ell_{\mathrm{pred}}=\mathcal N(\hat f(\tilde x), s^2(\tilde x))$ denote the classical Bayesian RF posterior predictive distribution induced by a Gaussian prior and Gaussian likelihood, and define the baseline predictive variance $V_{\mathrm{base}}(\psi_1;\lambda)
\coloneqq \mathbb{V}_{\ell_{\mathrm{pred}}}(\tilde Y).$

To highlight the dependence of the bounds for the variance on $\eta$, let 
\[
\underline V(\psi_1;\lambda,\eta)
\coloneqq \mathbb{V}_{\underline{\ell}_{\mathrm{pred}}}(\tilde{Y}),
\qquad
\overline V(\psi_1;\lambda,\eta)
\coloneqq \mathbb{V}_{\overline{\ell}_{\mathrm{pred}}}(\tilde{Y}).
\]

We now study formally the double descent phenomenon in the contamination credal set case.

\begin{corollary}[Persistence of Double Descent Under Contamination]
\label{cor:robust-dd-variance}

Fix $\psi_2 = n/d$ and $\lambda \ge 0$. Assume throughout that $\eta\in[0,1)$, so that $1-\eta>0$. Suppose that the baseline predictive variance $V_{\mathrm{base}}(\psi_1;\lambda)$ exhibits double descent as a function of $\psi_1$, with a nondegenerate local maximizer at $\psi_1^\star$.
Then the bounding functions appearing in Lemma~\ref{var-compl} preserve the location of this peak:

\begin{enumerate}
\item The function
\[
g_-(\psi_1)
\coloneqq
(1-\eta) V_{\mathrm{base}}(\psi_1;\lambda)
\]
has the same maximizer(s) as $V_{\mathrm{base}}(\psi_1;\lambda)$.

\item Under the truncation approximation with
\[
a(\psi_1)=\hat{f}(\tilde{\mathbf{x}})+\alpha_0\,s(\tilde{\mathbf{x}}),
\quad
b(\psi_1)=\hat{f}(\tilde{\mathbf{x}})+\beta_0\,s(\tilde{\mathbf{x}}),
\]
for fixed constants $\alpha_0<\beta_0$ independent of $\psi_1$, the function
\[
g_+(\psi_1):=(1-\eta)V'_{\mathrm{base}}(\psi_1;\lambda)+\eta\frac{(b-a)^2}{12},
\]
where $V'_{\mathrm{base}}(\psi_1;\lambda)$ is the predictive variance under
$\ell^\prime_{\mathrm{pred}}$, has the same maximizer(s) as $V_{\mathrm{base}}(\psi_1;\lambda)$.
\end{enumerate}

In particular, the bounding functions form an envelope around the double-descent peak $\psi_1^\star$.
\end{corollary}

\section{Numerical Experiments}\label{sec:experiments}


In this section, we investigate the mechanism underlying the double-descent phenomenon and its robustness to likelihood misspecification.

Specifically, we decompose the test error into bias and variance terms, and then study how likelihood contamination induces uncertainty envelopes around the variance.
The purpose of these experiments is deliberately targeted: rather than providing a broad empirical benchmark, we use simulations to illustrate the two theoretical mechanisms established above, namely variance-driven double descent and the preservation of the interpolation peak under contamination-induced uncertainty envelopes.
We emphasize that the proposed framework is not intended to improve point-prediction accuracy directly; its practical benefit is robust predictive uncertainty quantification, namely producing conservative predictive envelopes and adjusted uncertainty regions under misspecification.

\textbf{Bias-variance decomposition.} 
For each feature-to-dimension ratio $\psi_1 = N/d$, we estimate the decomposition
\begin{align*}
\mathbb{E}\big[(\hat f(x) - f^{*}(x))^2\big]=&\underbrace{\big(\mathbb{E}[\hat f(x)] - f^{*}(x)\big)^2}_{\text{Bias}^2}\\ &+
\underbrace{\mathbb{E}\big[(\hat f(x) - \mathbb{E}[\hat f(x)])^2\big]}_{\text{Variance}},
\end{align*}
where the expectation is taken over independent draws of the training data and random
features.
Figure~\ref{fig:bias-variance-contamination}.(left) shows that the double descent peak is
entirely driven by the variance term, while the bias varies smoothly with $\psi_1$.
This confirms that the interpolation instability in random-feature regression is a variance-dominated phenomenon, especially in the ridgeless case, consistent with existing random matrix analyses.

\textbf{Contamination-induced variance envelopes.}
We then visualize the effect of likelihood contamination on predictive uncertainty.
Lemma \ref{var-compl} shows that the lower and upper predictive variances under contamination level
$\eta$ satisfy affine bounds of the form
\[
\underline V(\psi_1;\lambda,\eta)
 \le 
(1-\eta) V_{\mathrm{base}}(\psi_1;\lambda)
 \le 
\overline V(\psi_1;\lambda,\eta),
\]
where $V_{\mathrm{base}}(\psi_1;\lambda)$ denotes the baseline predictive variance.
Corollary~\ref{cor:robust-dd-variance} further implies that such affine transformations preserve the location of
any double descent peak in $\psi_1$.

Figure~\ref{fig:bias-variance-contamination}.(right) illustrates contamination-induced variance envelopes
constructed via affine transformations of the baseline predictive variance. The lower envelope corresponds to the scaling $(1-\eta)V_{\mathrm{base}}$, while the upper envelope is constructed from the truncated variance bound of Proposition~\ref{prop-var2}, namely a term of the form $(1-\eta)m\,V'_{\mathrm{base}}+\eta(b-a)^2/12$, which under the truncation choice of Corollary~\ref{cor:robust-dd-variance} is itself proportional to the baseline variance.
These constructions reflect the affine structure of the theoretical
bounds in Lemma \ref{var-compl} and Corollary~\ref{cor:robust-dd-variance}, and demonstrate that contamination preserves
the location of the interpolation peak while amplifying predictive
dispersion near $N=n$.

\begin{figure*}[t]
    \centering
    \begin{minipage}{0.44\linewidth}
        \centering
        \includegraphics[width=\linewidth]{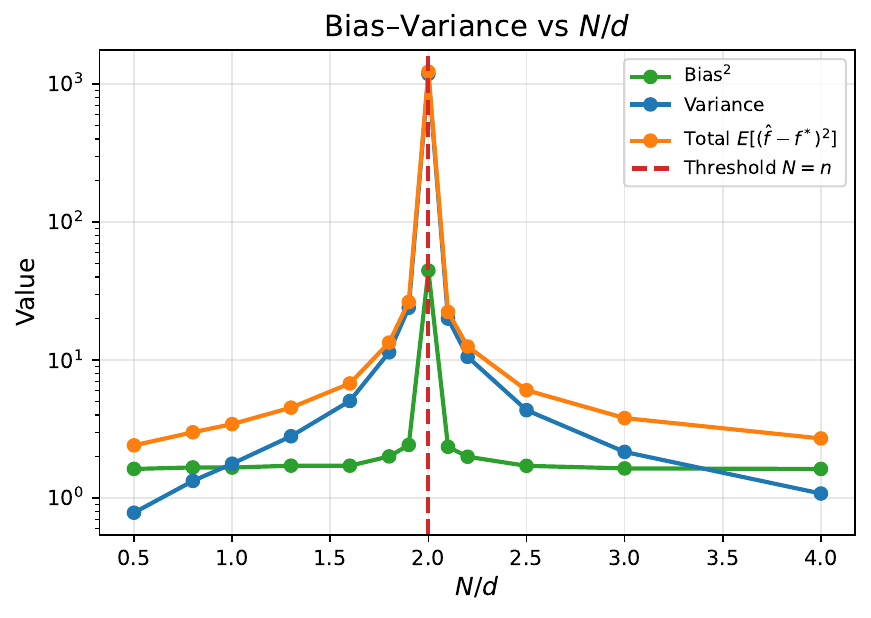}
        \caption*{(a) Bias-variance decomposition}
    \end{minipage}
    \hfill
    \begin{minipage}{0.44\linewidth}
        \centering
        \includegraphics[width=\linewidth]{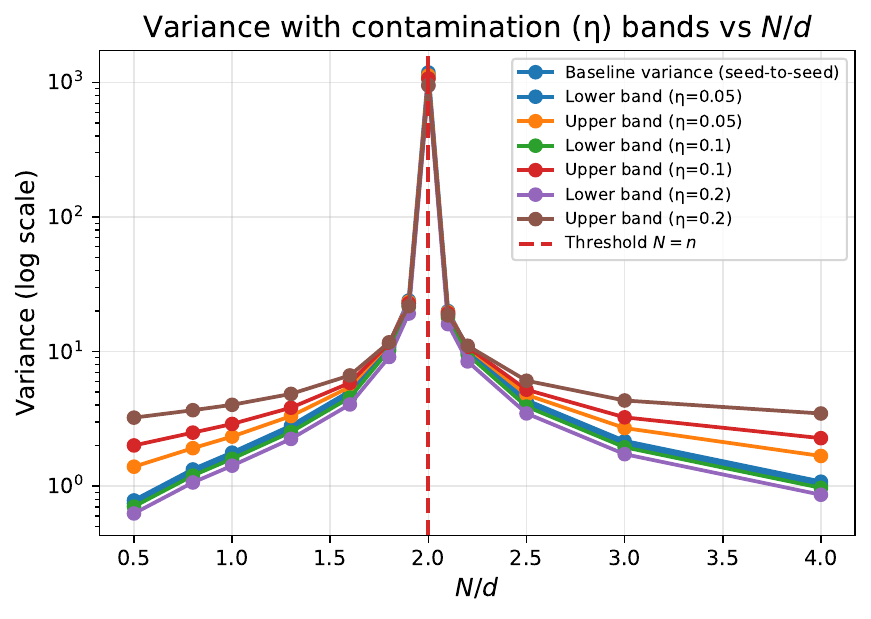}
        \caption*{(b) Variance with $\eta$-contamination bands}
    \end{minipage}
    \caption{Bias-variance mechanism and robustness to likelihood contamination.
    \emph{Left:} Bias$^2$, variance, and total error as functions of $\psi_1=N/d$,
    showing that the double descent peak is variance-driven.
    \emph{Right:} Contamination-induced variance envelopes for $\eta\in\{0.05,0.1,0.2\}$,
    illustrating the affine bounds of Lemma \ref{var-compl} and the persistence of the interpolation
    peak predicted by Corollary~\ref{cor:robust-dd-variance}.}
    \label{fig:bias-variance-contamination}
\end{figure*}


The likelihood contamination parameter $\eta$ introduced in Lemma~\ref{char_lemma} provides an abstract model of distributional misspecification, allowing a fraction of the
data-generating process to deviate arbitrarily from the assumed likelihood.
To connect this formal notion with concrete data corruption, we consider an explicit misspecification experiment based on label outliers.

\textbf{Experimental setup.}
We generate training data according to a random-feature teacher model,
\[
y = f^{*}(x) + \varepsilon,
\qquad \varepsilon \sim \mathcal N(0,\sigma^2),
\]
and introduce label misspecification via a Huber-type contamination model,
\begin{equation}
y =
\begin{cases}
f^*(x) + \varepsilon, & \text{with probability } 1-\eta, \\
f^*(x) + \varepsilon + A, & \text{with probability } \eta,
\end{cases}
\label{eq:huber-contamination}
\end{equation}
where $A \gg \sigma$ is a fixed outlier amplitude.
The parameter $\eta$ controls the fraction of corrupted labels.

For each contamination level $\eta$, we train a ridgeless random-feature regression model, i.e. we set $\lambda=0$, and evaluate the test mean squared error (MSE) as a function of the feature-to-dimension ratio $\psi_1 = N/d$, keeping the sample ratio $\psi_2 = n/d$ fixed.

\textbf{Results.}
Figure~\ref{fig:misspecification} reports the test MSE curves under increasing levels of label contamination.
In the absence of misspecification ($\eta=0$), the estimator exhibits the familiar double descent behavior, with a sharp error peak near the interpolation threshold $N=n$. As the contamination level $\eta$ increases, the height of the peak grows substantially, indicating strong noise amplification in the vicinity of interpolation.

Importantly, the location of the peak remains stable across contamination levels, while only its magnitude changes.
Although Lemma \ref{var-compl} and Corollary~\ref{cor:robust-dd-variance} concern predictive variance rather than test MSE directly, the present behavior is qualitatively consistent with those results, together with the bias-variance decomposition above, insofar as contamination primarily amplifies the variance-dominated interpolation peak without shifting its location in $\psi_1$.
The numerical results therefore support the interpretation of $\eta$ as a robustness parameter that controls the magnitude of uncertainty under misspecification, while preserving the structural double descent phenomenon. This experiment illustrates that the imprecision-based variance envelopes derived in Section~\ref{sec:variance} are not merely abstract worst-case bounds, but are qualitatively informative about the behavior of test error under concrete data corruption when the bias contribution remains comparatively stable.
\begin{figure}[t]
    \centering
    \includegraphics[width= .8\linewidth]{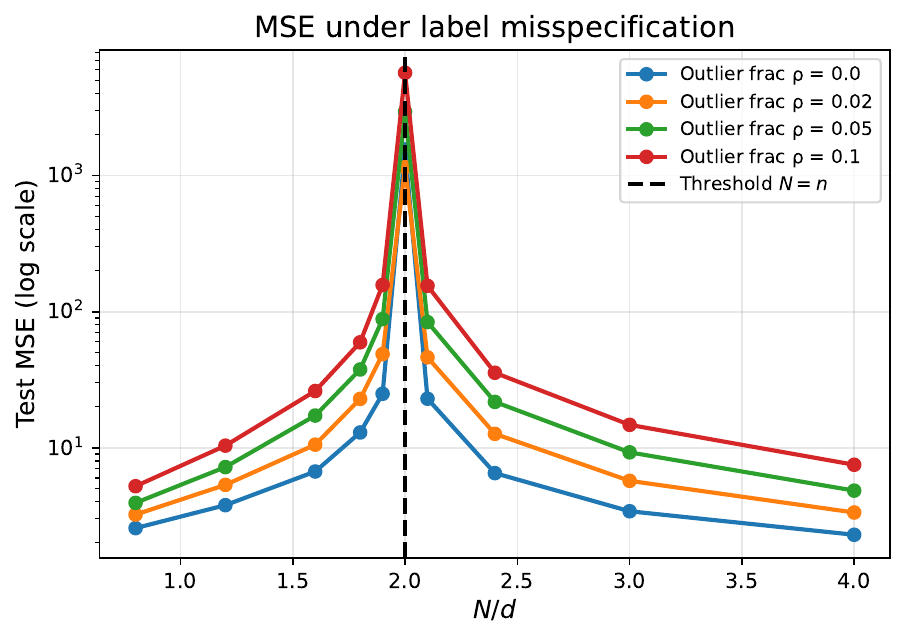}
    \caption{Test MSE as a function of $N/d$ under increasing levels of label
    misspecification $\eta$. Larger contamination amplifies the interpolation peak
    without shifting its location, consistent with Lemma \ref{var-compl} and Corollary~\ref{cor:robust-dd-variance}.}
    \label{fig:misspecification}
\end{figure}

\section{Conclusion}\label{concl}

We introduced a robust Bayesian formulation of random feature regression
based on Huber-style contamination sets for both prior and likelihood.
By replacing a single probabilistic specification with credal sets and
performing pessimistic generalized Bayesian updating, we obtained lower
and upper posterior predictive that explicitly quantify worst-case
predictive behavior under bounded model uncertainty.

Our main theoretical contribution is the derivation of simple and
tractable bounds on predictive densities and predictive variance.
These bounds act as affine transformations of the classical Bayesian
RF predictive distribution, yielding uncertainty envelopes that preserve
the leading-order proportional-growth asymptotics of random feature
models. In particular, the interpolation-driven double-descent structure
of predictive variance is retained, while its magnitude is adjusted to
reflect prior and likelihood contamination.

On the practical side, we introduced the Imprecise Highest Density
Region (IHDR) and showed that it admits an efficient approximation
via adjusted Gaussian credible intervals. This preserves the computational
tractability of Bayesian random features while providing robust guarantees.

Overall, our results establish a robustness theory for Bayesian random
features: predictive uncertainty remains analytically tractable and
asymptotically well-behaved, yet is improved by explicit worst-case
guarantees under bounded misspecification. Future work may investigate
data-adaptive choices of contamination levels, extensions to kernel
limits and deep feature models, and implications of
robust predictive uncertainty in high-dimensional regimes.

\appendix

\bibliography{tropical}

\appendix
\onecolumn
\section{Proofs of Our Results}\label{app-proofs}






\begin{proof}[Proof of Lemma \ref{low_up_dens}]
Let us begin with the lower density. Fix any $A\in \mathcal{B}(\mathbb{R}^N)\setminus\{\mathbb{R}^N\}$. By \citet[Example 3]{wasserman}, we know that $\underline{P}(A)=(1-\epsilon)P^\mathcal{N}(A)=(1-\epsilon)P^\mathcal{N}(A)+ \epsilon Q_{\star A}(A)$. In turn, we have that
\begin{align*}
    \underline{P}(A)&=(1-\epsilon)\int_A p^\mathcal{N}(\a) \text{d}\a + \epsilon \int_A q_{\star A}(\a) \text{d}\a\\
    &= \int_A \left[(1-\epsilon) p^\mathcal{N}(\a) + \epsilon q_{\star A}(\a) \right] \text{d}\a = \int_A \underline{p}(\a) \text{d}\a.
\end{align*}
For the upper density, fix any $A\in \mathcal{B}(\mathbb{R}^N)\setminus\{\emptyset\}$. By \citet[Example 3]{wasserman}, we know that $\overline{P}(A)=(1-\epsilon)P^\mathcal{N}(A) + \epsilon =(1-\epsilon)P^\mathcal{N}(A)+ \epsilon Q_{A}^\star(A)$. In turn, we have that
\begin{align*}
    \overline{P}(A)&=(1-\epsilon)\int_A p^\mathcal{N}(\a) \text{d}\a + \epsilon \int_A q_{A}^\star(\a) \text{d}\a\\
    &= \int_A \left[(1-\epsilon) p^\mathcal{N}(\a) + \epsilon q_{A}^\star(\a) \right] \text{d}\a = \int_A \overline{p}(\a)  \text{d}\a.
\end{align*}
\end{proof}

\begin{proof}[Proof of Theorem~\ref{main_th}]
    In what follows, we take $\underline{p}_{\a}$ to denote the $A$-independent lower envelope subdensity $(1-\epsilon)p^\mathcal{N}(\a)$. Throughout this proof, consistently with the absolutely continuous restriction discussed after Lemma~\ref{low_up_dens}, we restrict the likelihood contaminations to a subclass of bounded spike densities centered at the observed data $\mathbf y$, and we take this spike to be fixed, i.e. independent of the measurable set under consideration, of the level $t$ in the Choquet representation, and of $\a$. Hence, in the likelihood calculations below, we write
\(
\overline{\ell}_{\a}(\y)=(1-\eta)\ell^\mathcal{N}_{\a}(\y)+\eta \delta_{\mathbf y}(\y),
\)
where $\delta_{\mathbf y}$ is a bounded spike density centered at $\mathbf y$. Let $\underline{\ell}_{\a}(\y) \equiv \underline{p}(\y \mid \X,\bT,\a)$, and similarly for $\overline{\ell}_{\a}(\y)$. Under the absolutely continuous restriction discussed after Lemma~\ref{low_up_dens}, we work with the following density-level representation of the pessimistic generalized Bayes update \eqref{pessim-gen-bayes},
    \begin{align}\label{gen-bayes}
        \underline{p}(\a \mid \y, \X,\bT)=\frac{\underline{p}(\a) \underline{\ell}_{\a}(\y)}{\sup_{P\in\mathcal{P}_\epsilon} \int \overline{\ell}_{\a}(\y) P(d\a)}.
    \end{align}
    We also have that 
    $\overline{P}$ 
    is $2$-alternating, i.e. $\overline{P}(A \cup B) \leq \overline{P}(A)+\overline{P}(B)-\overline{P}(A \cap B)$, for all $A,B \in \mathcal{B}(\mathbb{R}^N)$ \citep{wasserman,novel_Bayes}. This means that the upper expectation $\sup_{P\in\mathcal{P}_\epsilon} \int \overline{\ell}_{\a}(\y) P(d\a)$ coincides with the Choquet integral $\int \overline{\ell}_{\a}(\y) \overline{P}(d\a)$ \citep{cerreia}, making computations easier. Let us focus on such an integration. We have that
    \begin{align}
        \int \overline{\ell}_{\a}(\y) \overline{P}(d\a) &= \int_0^\infty \overline{P} \left( \left\lbrace{\a : \overline{\ell}_{\a}(\y) \geq t}\right\rbrace \right) \text{d}t \label{first_deriv}\\
        &= \int_0^\lamed \overline{P} \left( \left\lbrace{\a : \overline{\ell}_{\a}(\y) \geq t}\right\rbrace \right) \text{d}t \label{second_deriv}\\
        &= \int_0^\lamed \left[ (1-\epsilon) P^\mathcal{N}\left( \left\lbrace{\a : \overline{\ell}_{\a}(\y) \geq t}\right\rbrace \right) + \epsilon Q^\star_{\left\lbrace{\a : \overline{\ell}_{\a}(\y) \geq t}\right\rbrace} \left( \left\lbrace{\a : \overline{\ell}_{\a}(\y) \geq t}\right\rbrace \right) \right] \text{d}t \label{third_deriv}\\
        &=(1-\epsilon)\int_0^\lamed P^\mathcal{N}\left( \left\lbrace{\a : \overline{\ell}_{\a}(\y) \geq t}\right\rbrace \right) \text{d}t  + \epsilon \int_0^\lamed \underbrace{Q^\star_{\left\lbrace{\a : \overline{\ell}_{\a}(\y) \geq t}\right\rbrace} \left( \left\lbrace{\a : \overline{\ell}_{\a}(\y) \geq t}\right\rbrace \right)}_{=1 \text{ for all } t} \text{d}t \label{fourth_deriv}\\
        &=(1-\epsilon) \int \overline{\ell}_{\a}(\y) p^\mathcal{N}(\a) \text{d}\a + \underbrace{\epsilon \int_0^\lamed \text{d}t}_{= \epsilon\cdot \lamed \eqqcolon c}\nonumber\\
        &=(1-\epsilon) \int \left[ (1-\eta) \ell^\mathcal{N}_{\a}(\y) + \eta \delta_{\y} \right] p^\mathcal{N}(\a) \text{d}\a +c \label{fifth_deriv}\\
        &=(1-\epsilon)(1-\eta) \int \ell^\mathcal{N}_{\a}(\y) p^\mathcal{N}(\a) \text{d}\a + (1-\epsilon)\eta \delta_{\y} \int p^\mathcal{N}(\a) \text{d}\a +c \nonumber\\
        &=(1-\epsilon)(1-\eta)p(\y \mid \X,\bT) + (1-\epsilon)\eta \delta_{\y} + c, \label{sixth_deriv}
    \end{align}
where \eqref{first_deriv} comes from \citet[Proposition C.3]{decooman} and \citet[Equation (11)]{marinacci2}; \eqref{second_deriv} comes from $\overline{\ell}_{\a}(\y)$ being bounded and having set $\lamed \coloneqq \sup_{\a} \overline{\ell}_{\a}(\y)$; \eqref{third_deriv} comes from Lemma \ref{low_up_dens}; \eqref{fourth_deriv} comes from the additivity of the integral operator and Lemma \ref{low_up_dens}; in \eqref{fifth_deriv}, the spike density term $\delta_{\y}$ is at the observed data $\mathbf{y}$; \eqref{sixth_deriv} comes from $p(\y \mid \X,\bT)=\int \ell^\mathcal{N}_{\a}(\y) p^\mathcal{N}(\a) \text{d}\a$ being the marginal likelihood in the classical Bayesian RF model, and $\delta_{\y}\int p^\mathcal{N}(\a) d\a=\delta_{\y}$ because $\delta_{\y}$ does not depend on $\a$, and $\int p^\mathcal{N}(\a) d\a=1$
because $p^\mathcal{N}(\a)$ is a proper density. 

In turn, by Lemma \ref{low_up_dens} and by substituting \eqref{sixth_deriv} in \eqref{gen-bayes}, we obtain

$$\underline{p}(\a \mid \y, \X,\bT)=\frac{(1-\epsilon)p^\mathcal{N}(\a) (1-\eta) \ell^\mathcal{N}_{\a}(\y)}{(1-\epsilon)(1-\eta)\int \ell^\mathcal{N}_{\a}(\y) p^\mathcal{N}(\a) \text{d}\a + (1-\epsilon)\eta \delta_{\y} + c}.$$

For proper measurable sets $A\in\mathcal{B}(\mathbb{R}^N)\setminus\{\mathbb{R}^N\}$, this expression is understood as the subdensity representation associated with the updated lower probability, while for $A=\mathbb{R}^N$ one must appeal to the proper setwise characterization recalled in Lemma~\ref{char_lemma}. In particular, in the derivation of the predictive bound below, we use this subdensity representation as an auxiliary density-level device for bounding the lower posterior predictive envelope, rather than as a globally normalized conditional density on all of $\mathbb{R}^N$. It can be routinely checked that, on proper measurable sets and together with the full-space convention from Lemma~\ref{char_lemma}, the corresponding updated set function defines a well-behaved lower probability. That is, for all $B \in \mathcal{B}(\mathbb{R}^n)$, it satisfies 

\begin{itemize}
    \item[(i)] $\underline{P}(\a \in \emptyset \mid \X,\bT,\y\in B)=0$,
    \item[(ii)] $\underline{P}(\a \in \mathbb{R}^N \mid \X,\bT,\y\in B)=1$,
    \item[(iii)] $A_1 \subseteq A_2 \implies \underline{P}(\a \in A_1 \mid \X,\bT,\y\in B) \leq \underline{P}(\a \in A_2 \mid \X,\bT,\y\in B)$,
    \item[(iv)] $\underline{P}(\a \in A_1 \sqcup A_2 \mid \X,\bT,\y\in B) \geq \underline{P}(\a \in A_1 \mid \X,\bT,\y\in B) + \underline{P}(\a \in A_2 \mid \X,\bT,\y\in B)$. 
\end{itemize}


Suppose now we are given a new input $\tilde{\x}$. We then focus on the lower posterior predictive density $\underline{\ell}_\text{pred}(\tilde{y})\equiv \underline{p}(\tilde{y} \mid \tilde{\x},\y,\X,\bT)$. Let us denote the likelihood $\underline{\ell}_{\a}(\tilde{y}) \equiv \underline{p}(\tilde{y} \mid \tilde{\x}, \bT,\a)$ for ease of notation. Using a reasoning similar to that above, we write 

\begin{align}
   \underline{\ell}_\text{pred}(\tilde{y}) &= \inf_{P \in \mathcal{M}(\underline{P}(\cdot \mid \X,\bT,\y\in B))} \int \underline{\ell}_{\a}(\tilde{y}) P(\text{d}\a \mid \X,\bT,\y\in B) \nonumber\\
   &= \int \underline{\ell}_{\a}(\tilde{y}) \underline{P}(\text{d}\a \mid \X,\bT,\y\in B) \label{first-part2}\\
   &= \int \underline{\ell}_{\a}(\tilde{y}) \underline{p}(\a \mid \y, \X,\bT) \text{d}\a \label{second-part2}\\
   &=\int (1-\eta) \ell^\mathcal{N}_{\a}(\tilde{y})
\left[ \frac{(1-\epsilon)p^\mathcal{N}(\a) (1-\eta) \ell^\mathcal{N}_{\a}(\y)}{(1-\epsilon)(1-\eta)\int \ell^\mathcal{N}_{\a}(\y) p^\mathcal{N}(\a) \text{d}\a + (1-\epsilon)\eta \delta_{\y} + c} \right] \text{d}\a \nonumber\\
   &= \int \frac{(1-\epsilon)(1-\eta)^2 p^\mathcal{N}(\a) \ell^\mathcal{N}_{\a}(\y) \ell^\mathcal{N}_{\a}(\tilde{y})}{(1-\epsilon)(1-\eta)\int \ell^\mathcal{N}_{\a}(\y) p^\mathcal{N}(\a) \text{d}\a +  (1-\epsilon)\eta \delta_{\y} + c} \text{d}\a \nonumber\\
   &= \frac{(1-\epsilon)(1-\eta)^2}{(1-\epsilon)(1-\eta)\int \ell^\mathcal{N}_{\a}(\y) p^\mathcal{N}(\a) \text{d}\a +  (1-\epsilon)\eta \delta_{\y} + c} \int p^\mathcal{N}(\a) \ell^\mathcal{N}_{\a}(\y) \ell^\mathcal{N}_{\a}(\tilde{y}) \text{d}\a \nonumber\\
   &= \frac{(1-\epsilon)(1-\eta)^2 \int \ell^\mathcal{N}_{\a}(\y) p^\mathcal{N}(\a) \text{d}\a}{(1-\epsilon)(1-\eta)\int \ell^\mathcal{N}_{\a}(\y) p^\mathcal{N}(\a) \text{d}\a +  (1-\epsilon)\eta \delta_{\y} + c} \int \ell^\mathcal{N}_{\a}(\tilde{y}) \underbrace{\frac{p^\mathcal{N}(\a) \ell^\mathcal{N}_{\a}(\y)}{\int p^\mathcal{N}(\a) \ell^\mathcal{N}_{\a}(\y) \text{d}\a}}_{\text{this ratio is }=p^\mathcal{N}(\a \mid \y, \X,\bT)}  \text{d}\a \label{third-part2}\\
   &= \frac{(1-\epsilon)(1-\eta)^2 \int \ell^\mathcal{N}_{\a}(\y) p^\mathcal{N}(\a) \text{d}\a}{(1-\epsilon)(1-\eta)\int \ell^\mathcal{N}_{\a}(\y) p^\mathcal{N}(\a) \text{d}\a +  (1-\epsilon)\eta \delta_{\y} + c} \underbrace{p(\tilde{y} \mid \tilde{\x},\y,\X,\Theta)}_{\equiv \ell_\text{pred}(\tilde{y})} \label{fourth-part2}\\
   &\leq (1-\eta) \ell_\text{pred}(\tilde{y}), \label{fifth-part2}
\end{align}
where \eqref{first-part2} comes from the posterior lower probability $\underline{P}(\cdot \mid \X,\bT,\y\in B)$ being 2-monotone \citep{wasserman,novel_Bayes}; 
\eqref{second-part2} comes from writing the Choquet integral in terms of the superlevel sets of $\underline{\ell}_{\a}(\tilde y)$, namely the sets
\(
\{\a\in\mathbb{R}^N:\underline{\ell}_{\a}(\tilde y)\ge t\},
\)
$t >0$, and observing that these are proper measurable subsets of $\mathbb{R}^N$; in \eqref{third-part2}, we denoted by $p^\mathcal{N}(\a \mid \y, \X,\bT)$ the (classical) posterior derived from \eqref{normal_pr} and \eqref{normal_lik}; in \eqref{fourth-part2}, we denoted by $p(\tilde{y} \mid \tilde{\x},\y,\X,\Theta)$ the (classical) posterior predictive distribution, derived from \eqref{normal_pr} and \eqref{normal_lik}, and the new input $\tilde{\x}$. 


Equation \eqref{fifth-part2} gives an upper bound for $\underline{\ell}_\text{pred}(\tilde{y})$.
Moreover, if $(1-\epsilon)\eta \delta_{\y}+c$ is negligible compared to
$(1-\epsilon)(1-\eta)\int \ell^\mathcal{N}_{\a}(\y)p^\mathcal{N}(\a) d\a$,
then $\underline{\ell}_\text{pred}(\tilde{y}) \approx (1-\eta)\ell_\text{pred}(\tilde{y})$.
\end{proof}

\begin{proof}[Proof of Corollary \ref{cor-1}]
   Similarly to the proof of Theorem~\ref{main_th}, under the absolutely continuous bounded-spike subclass discussed after Lemma~\ref{low_up_dens}, and restricting attention to contaminating prior densities for which a global upper density representation exists, we work with the following density-level representation of the optimistic generalized Bayes update, 
    \begin{align}\label{gen-bayes2}
        \overline{p}(\a \mid \y, \X,\bT)=\frac{\overline{p}(\a) \overline{\ell}_{\a}(\y)}{\inf_{P\in\mathcal{P}_\epsilon} \int \underline{\ell}_{\a}(\y) P(d\a)},
    \end{align} 
where $\overline{p}(\a)$ denotes the $A$-independent upper density, when it exists.
    
Fix any probability density $u_n$ on $\mathbb{R}^n$ and the density $u$ on $\mathbb{R}$ appearing in the statement of Corollary~\ref{cor-1}.
Throughout this proof, interpret the likelihood contamination term in $\overline{\ell}_{\a}(\y)$ as $\eta u_n(\y)$, and the one in $\overline{\ell}_{\a}(\tilde y)$ as $\eta u(\tilde y)$.

    Let us focus on the integral at the denominator. We know that $\underline{P}\coloneqq \inf_{P\in\mathcal{P}_\epsilon}P$ is $2$-monotone, i.e. $\underline{P}(A \cup B) \geq \underline{P}(A)+\underline{P}(B)-\underline{P}(A \cap B)$, for all $A,B \in \mathcal{B}(\mathbb{R}^N)$ -- see e.g. \citet{wasserman,novel_Bayes}. This means that the lower expectation $\inf_{P\in\mathcal{P}_\epsilon} \int \underline{\ell}_{\a}(\y) P(d\a)$ coincides with the Choquet integral $\int \underline{\ell}_{\a}(\y) \underline{P}(d\a)$ \citep{cerreia}. In turn, we have that
\begin{align*}
        \int \underline{\ell}_{\a}(\y) \underline{P}(d\a) &= \int_0^\infty \underline{P} \left( \left\lbrace{\a : \underline{\ell}_{\a}(\y) \geq t}\right\rbrace \right) \text{d}t\\
        &= \int_0^{\lamed^\prime} \left[ (1-\epsilon) P^\mathcal{N}\left( \left\lbrace{\a : \underline{\ell}_{\a}(\y) \geq t}\right\rbrace \right) + \epsilon Q_{\star\left\lbrace{\a : \underline{\ell}_{\a}(\y) \geq t}\right\rbrace} \left( \left\lbrace{\a : \underline{\ell}_{\a}(\y) \geq t}\right\rbrace \right) \right] \text{d}t \\
        &=(1-\epsilon)\int_0^{\lamed^\prime} P^\mathcal{N}\left( \left\lbrace{\a : \underline{\ell}_{\a}(\y) \geq t}\right\rbrace \right) \text{d}t  + \epsilon \int_0^{\lamed^\prime} \underbrace{Q_{\star\left\lbrace{\a : \underline{\ell}_{\a}(\y) \geq t}\right\rbrace} \left( \left\lbrace{\a : \underline{\ell}_{\a}(\y) \geq t}\right\rbrace \right)}_{=0 \text{ for all } t} \text{d}t\\
        &=(1-\epsilon) \int \underline{\ell}_{\a}(\y) p^\mathcal{N}(\a) \text{d}\a \\
        &=(1-\epsilon) \int (1-\eta) \ell^\mathcal{N}_{\a}(\y)  p^\mathcal{N}(\a)  \text{d}\a\\
        &=(1-\eta) (1-\epsilon) \underbrace{\int \ell^\mathcal{N}_{\a}(\y) p^\mathcal{N}(\a) \text{d}\a}_{=p(\y \mid \X,\bT)},
    \end{align*}
where $\lamed^\prime \coloneqq \sup_{\a} \underline{\ell}_{\a}(\y)$. So, we can rewrite \eqref{gen-bayes2} as
\begin{align*}
    \overline{p}(\a \mid \y, \X,\bT)&=\frac{\overline{p}(\a) \overline{\ell}_{\a}(\y)}{(1-\eta) (1-\epsilon) \int \ell^\mathcal{N}_{\a}(\y) p^\mathcal{N}(\a) \text{d}\a}\\
    &\ge \frac{(1-\epsilon) p^\mathcal{N}(\a)  \left[ (1-\eta) \ell^\mathcal{N}_{\a}(\y) +\eta u_n(\y) \right]}{(1-\eta) (1-\epsilon) \int \ell^\mathcal{N}_{\a}(\y) p^\mathcal{N}(\a) \text{d}\a}\\
    &\geq \frac{ (1-\eta)(1-\epsilon) p^\mathcal{N}(\a) \ell^\mathcal{N}_{\a}(\y) }{(1-\eta) (1-\epsilon) \int \ell^\mathcal{N}_{\a}(\y) p^\mathcal{N}(\a) \text{d}\a} = {p}(\a \mid \y, \X,\bT).
\end{align*}

Since optimistic generalized Bayes preserves concavity of upper probabilities, the updated upper probability is $2$-alternating \citep{wasserman,novel_Bayes}. Hence, interpreting the upper posterior predictive prevision as the natural extension of the updated upper probability, it coincides with the Choquet integral with respect to that updated upper probability. For proper measurable sets $A\in\mathcal{B}(\mathbb{R}^N)\setminus\{\emptyset\}$, the expression above is understood as the superdensity representation associated with the updated upper probability, while for $A=\emptyset$ one appeals to the proper setwise characterization in Lemma~\ref{char_lemma}. In particular, since for every $t>0$ the level sets
\(
\{\a\in\mathbb{R}^N:\overline{\ell}_{\a}(\tilde y)\ge t\}
\)
are proper measurable subsets of $\mathbb{R}^N$, this superdensity representation is sufficient for the Choquet calculation below; hence, with a slight abuse of notation, we may write the corresponding Choquet integral in superdensity form. We can then turn our attention to the upper posterior predictive density. We have that

\begin{align*}
    \overline{\ell}_\text{pred}(\tilde{y}) &= \int \overline{\ell}_{\a}(\tilde{y}) \overline{p}(\a \mid \y, \X,\bT) \text{d}\a \\
    &\geq \int \overline{\ell}_{\a}(\tilde{y}) {p}(\a \mid \y, \X,\bT) \text{d}\a\\
    &= \int \left[ (1-\eta) {\ell}^\mathcal{N}_{\a}(\tilde{y}) + \eta u(\tilde{y}) \right] {p}(\a \mid \y, \X,\bT) \text{d}\a\\
&= (1-\eta) \int {\ell}^\mathcal{N}_{\a}(\tilde{y}) {p}(\a \mid \y, \X,\bT) \text{d}\a  +  \eta u(\tilde{y}) \underbrace{\int {p}(\a \mid \y, \X,\bT) \text{d}\a}_{=1}\\
&= (1-\eta) \underbrace{p^\mathcal{N}(\tilde{y} \mid \tilde{\x},\y,\X,\bT)}_{\equiv \ell_\text{pred}(\tilde{y})}  +  \eta u(\tilde{y}),
\end{align*}
where the first equality is the Choquet integral written in superdensity form, justified by the preceding paragraph, and the inequality follows from the pointwise bound $\overline{p}(\a \mid \y,\X,\bT)\ge p(\a \mid \y,\X,\bT)$ together with the nonnegativity of $\overline{\ell}_{\a}(\tilde y)$.
The last equality gives a lower bound for $\overline{\ell}_\text{pred}(\tilde y)$.
Moreover, for $\epsilon,\eta$ sufficiently small, the upper posterior can be well approximated by the classical posterior whenever the contaminating contributions are negligible. 
Indeed, for any choice of contaminating prior density $q$ and contaminating likelihood density $u_n$, write
\[
\overline p(\a)=(1-\epsilon)p^{\mathcal N}(\a)+\epsilon q(\a),
\qquad
\overline \ell_{\a}(\y)=(1-\eta)\ell^{\mathcal N}_{\a}(\y)+\eta u_n(\y).
\]
Then,
\[
\overline p(\a) \overline \ell_{\a}(\y)
=(1-\epsilon)(1-\eta)p^{\mathcal N}(\a)\ell^{\mathcal N}_{\a}(\y)
+\epsilon(1-\eta)q(\a)\ell^{\mathcal N}_{\a}(\y)
+(1-\epsilon)\eta p^{\mathcal N}(\a)u_n(\y)
+\epsilon\eta q(\a)u_n(\y).
\]
If the last three terms are small compared to the first one (e.g. in $L^1(d\a)$ over $\a$), then the resulting upper posterior is close to the classical posterior, and consequently
$\overline{\ell}_\text{pred}(\tilde y)\approx (1-\eta)\ell_\text{pred}(\tilde y)+\eta u(\tilde y)$.
\end{proof}

\begin{proof}[Proof of Proposition~\ref{char-ihdr}]
    Recall the definition of a $\beta$-level (precise) Highest Density Region  for the probability measure $P_\text{pred}$ associated with pdf $\ell_\text{pred}$, denoted by $R_\beta[P_\text{pred}]$ \citep{hyndman}. It is the element of $\mathcal{B}(\mathbb{R})$ that satisfies two properties, (i) ${P}_\text{pred}\left(\{\tilde{y}\in R_\beta[P_\text{pred}]\}\right)= \beta$, and (ii) $\int_{R_\beta[P_\text{pred}]} \mathrm{d}\tilde{y}$ is a minimum.
    Let
    \[
    S\coloneqq \text{IR}_\alpha[\mathcal{M}(\underline{P}_\text{pred})].
    \]
    If $\alpha=0$, then necessarily $S=\mathbb{R}$ and $\beta=1$, so $R_\beta[P_\text{pred}]=\mathbb{R}$ and the claim is immediate. Hence, in the remainder of the proof, assume $0<\alpha\le 1$, so that $S$ is a proper measurable set. By definition of IHDR,
    \[
    \underline{P}_\text{pred}\left(\{\tilde{y}\in S\}\right)=1-\alpha.
    \]
    By Theorem~\ref{main_th},
    \begin{align}\label{intermezzo}
\underline{P}_\text{pred}\left(\{\tilde{y}\in S\}\right)
    = \int_S \underline{\ell}_\text{pred}(\tilde{y})  \mathrm{d}\tilde{y}
    \leq (1-\eta) \int_S {\ell}_\text{pred}(\tilde{y})  \mathrm{d}\tilde{y}
    = (1-\eta) P_\text{pred}\left(\{\tilde{y}\in S\}\right).
    \end{align}
    Since $P_\text{pred}(S)\le 1$, \eqref{intermezzo} implies
\[
1-\alpha=\underline{P}_\text{pred}(S)\le (1-\eta)P_\text{pred}(S)\le 1-\eta,
\]
hence necessarily $\alpha\ge \eta$. Therefore, under the existence assumption for the IHDR and for $\eta<1$, we have
\[
\beta=\min\left\{1,\frac{1-\alpha}{1-\eta}\right\}=\frac{1-\alpha}{1-\eta}.
\]
    
    In turn,
    \[
    P_\text{pred}\left(\{\tilde{y}\in S\}\right)\geq \frac{1-\alpha}{1-\eta}\eqqcolon \beta.
    \]
    Thus $S$ is feasible for the optimization problem defining $R_\beta[P_\text{pred}]$. By minimality of $R_\beta[P_\text{pred}]$,
    \[
    \int_{R_\beta[P_\text{pred}]} \mathrm{d}\tilde{y}
    \le
    \int_S \mathrm{d}\tilde{y}.
    \]
    Since $\ell_\text{pred}$ is a Normal distribution, it is symmetric and unimodal, and hence $R_\beta[P_\text{pred}]$ coincides with the classical notion of credible interval \citep{hyndman,ibnn}.
\end{proof}

\begin{proof}[Proof of Proposition~\ref{prop-var}]
    We have that
    \begin{align}
        \mathbb{V}_{\underline{\ell}_\text{pred}}(\tilde{Y}) &\coloneqq \min_{\zeta\in\mathbb{R}} \mathbb{E}_{\underline{\ell}_\text{pred}}[(\tilde{Y}-\zeta)^2] \nonumber \\
        &= \min_{\zeta\in\mathbb{R}} \int_\mathbb{R} (\tilde{y}-\zeta)^2 \underline{\ell}_\text{pred}(\tilde{y})\text{d}\tilde{y} \nonumber \\
        &\leq (1-\eta) \min_{\zeta\in\mathbb{R}} \int_\mathbb{R} (\tilde{y}-\zeta)^2 {\ell}_\text{pred}(\tilde{y})\text{d}\tilde{y} \label{thm-3-conseq}\\
        &= (1-\eta) \mathbb{V}_{{\ell}_\text{pred}}(\tilde{Y}), \nonumber
    \end{align}
    where \eqref{thm-3-conseq} is an immediate consequence of Theorem \ref{main_th}, since for every fixed $\zeta\in\mathbb R$ the function $\tilde y\mapsto (\tilde y-\zeta)^2$ is nonnegative.
\end{proof}

\begin{proof}[Proof of Proposition~\ref{prop-var2}]
    We have that
    \begin{align}
        \mathbb{V}_{\overline{\ell}_\text{pred}}(\tilde{Y})
        &\coloneqq \min_{\zeta\in\mathbb{R}} \mathbb{E}_{\overline{\ell}_\text{pred}}[(\tilde{Y}-\zeta)^2] \nonumber \\
        &= \min_{\zeta\in\mathbb{R}} \int_\mathbb{R} (\tilde{y}-\zeta)^2 \overline{\ell}_\text{pred}(\tilde{y})\,\text{d}\tilde{y} \nonumber \\
        &\geq \min_{\zeta\in\mathbb{R}} \int_\mathbb{R} (\tilde{y}-\zeta)^2
        \left[(1-\eta) \ell_\text{pred}(\tilde{y}) + \eta u_{[a,b]}(\tilde{y}) \right]
        \text{d}\tilde{y} \label{cor-3.1-conseq-rev}\\
        &\geq \min_{\zeta\in\mathbb{R}} \left[
        (1-\eta)\int_a^b (\tilde{y}-\zeta)^2 \ell_\text{pred}(\tilde{y})\,\text{d}\tilde{y}
        +
        \eta \int_a^b (\tilde{y}-\zeta)^2 \frac{1}{b-a}\,\text{d}\tilde{y}
        \right], \label{restrict-int}
    \end{align}
    where \eqref{cor-3.1-conseq-rev} follows from Corollary~\ref{cor-1}, and \eqref{restrict-int} follows by restricting the first integral to $[a,b]$, since the integrand is nonnegative.

    By definition of the truncated density, for $\tilde y\in[a,b]$ we have
    \[
    \ell_\text{pred}(\tilde y)=m\,\ell_\text{pred}^\prime(\tilde y).
    \]
    Therefore,
    \begin{align}
        \mathbb{V}_{\overline{\ell}_\text{pred}}(\tilde{Y})
        &\geq \min_{\zeta\in\mathbb{R}} \left[
        (1-\eta)m\int_a^b (\tilde{y}-\zeta)^2 \ell_\text{pred}^\prime(\tilde{y})\,\text{d}\tilde{y}
        +
        \eta \int_a^b (\tilde{y}-\zeta)^2 \frac{1}{b-a}\,\text{d}\tilde{y}
        \right] \nonumber \\
        &\geq (1-\eta)m \min_{\zeta\in\mathbb{R}} \int_a^b (\tilde{y}-\zeta)^2 \ell_\text{pred}^\prime(\tilde{y})\,\text{d}\tilde{y}
        +
        \eta \min_{\zeta\in\mathbb{R}} \int_a^b (\tilde{y}-\zeta)^2 \frac{1}{b-a}\,\text{d}\tilde{y} \nonumber \\
        &= (1-\eta)m\, \mathbb{V}_{{\ell}_\text{pred}^\prime}(\tilde{Y}) + \eta \frac{(b-a)^2}{12}. \label{final-eq-result-rev}
    \end{align}
    The last equality holds because the value $\zeta=\frac{a+b}{2}$ that minimizes
    \[
    \int_a^b (\tilde{y}-\zeta)^2 \frac{1}{b-a}\,\text{d}\tilde{y}
    \]
    is the midpoint of the interval $[a,b]$, and the corresponding minimum is the variance of the uniform distribution on $[a,b]$, namely $(b-a)^2/12$.

    Finally, if $m\approx 1$, that is, if $[a,b]$ captures most of the Gaussian mass of $\ell_\text{pred}$, then \eqref{final-eq-result-rev} yields the approximation
    \[
    \mathbb{V}_{\overline{\ell}_\text{pred}}(\tilde{Y}) \gtrsim
    (1-\eta)\mathbb{V}_{{\ell}_\text{pred}^\prime}(\tilde{Y}) + \eta \frac{(b-a)^2}{12}.
    \]
\end{proof}

We now present a lemma that will play a pivotal role in the proof of Lemma \ref{var-compl}.

\begin{lemma}[Predictive Variance Bounds: Sanity Check]\label{sanity}
    It holds that $\mathbb{V}_{\underline{\ell}_\text{pred}}(\tilde{Y}) \leq \mathbb{V}_{{\ell}_\text{pred}}(\tilde{Y}) \leq \mathbb{V}_{\overline{\ell}_\text{pred}}(\tilde{Y})$.
\end{lemma}


\begin{proof}[Proof of Lemma \ref{sanity}]
By \citet[Appendix G]{walley}, we can write
\[
\mathbb{V}_{{\ell}_\text{pred}}(\tilde{Y})=\min_{\zeta\in\mathbb{R}} \mathbb{E}_{{\ell}_\text{pred}}[(\tilde{Y}-\zeta)^2].
\]
Moreover, by Theorem~\ref{main_th},
\[
\underline{\ell}_\text{pred}(\tilde y)\le (1-\eta)\ell_\text{pred}(\tilde y)\le \ell_\text{pred}(\tilde y),
\]
and by Corollary~\ref{cor-1}, choosing $u=\ell_\text{pred}$, we obtain
\[
\overline{\ell}_\text{pred}(\tilde y)\ge (1-\eta)\ell_\text{pred}(\tilde y)+\eta \ell_\text{pred}(\tilde y)=\ell_\text{pred}(\tilde y).
\]
Hence, for each $\zeta\in\mathbb{R}$,
\[
\mathbb{E}_{\underline{\ell}_\text{pred}}[(\tilde{Y}-\zeta)^2]\le
\mathbb{E}_{{\ell}_\text{pred}}[(\tilde{Y}-\zeta)^2]\le
\mathbb{E}_{\overline{\ell}_\text{pred}}[(\tilde{Y}-\zeta)^2],
\]
because $(\tilde y-\zeta)^2\ge 0$. Minimizing over $\zeta$ yields the claim.
\end{proof}

\begin{proof}[Proof of Lemma \ref{var-compl}]
    The first inequality in \eqref{var-bds} comes from Proposition \ref{prop-var}, the second from $\eta$ being in $[0,1]$, the third from our assumption that $\frac{(b-a)^2}{12} \geq \mathbb{V}_{{\ell}_\text{pred}}(\tilde{Y})$, the following approximation from our assumption ${\ell}_\text{pred} \approx {\ell}_\text{pred}^\prime$,
and the fourth from Proposition \ref{prop-var2} together with the assumption $m \approx 1$.
\end{proof}

Since the contamination bounds act as affine transformations of the baseline predictive variance, we first record a simple invariance property of extrema under positive affine transformations.

\begin{lemma}[Affine Invariance of Extrema]\label{lem:affine-invariance}
Let $f:I\to\mathbb{R}$ be a real-valued function on an interval $I\subset\mathbb{R}$,
and let $a>0$, $b\in\mathbb{R}$.
Define $g(x)\coloneqq a f(x)+b$.
Then,
\begin{align*}
    \arg\max_{x\in I} g(x)&=\arg\max_{x\in I} f(x),\\
\arg\min_{x\in I} g(x)&=\arg\min_{x\in I} f(x).
\end{align*}
\end{lemma}

\begin{proof}
For any $x_1,x_2\in I$,
\[
f(x_1)\le f(x_2)
\quad\Longleftrightarrow\quad
a f(x_1)+b\le a f(x_2)+b,
\]
since $a>0$. Thus the ordering of function values is preserved,
and the locations of maxima and minima coincide.
\end{proof}

\begin{proof}[Proof of Corollary \ref{cor:robust-dd-variance}]
Since $\eta<1$, we have $1-\eta>0$, and therefore Lemma~\ref{lem:affine-invariance} applies to $g_-$. The function $g_-(\psi_1)=(1-\eta)V_{\mathrm{base}}(\psi_1;\lambda)$
is an affine transformation of $V_{\mathrm{base}}$ with positive
multiplicative constant $1-\eta$. By Lemma~\ref{lem:affine-invariance},
$g_-$ and $V_{\mathrm{base}}$ share the same maximizer(s).

Similarly, under the truncation approximation with
\[
a(\psi_1)=\hat{f}(\tilde{\mathbf{x}})+\alpha_0 s(\tilde{\mathbf{x}}),
\qquad
b(\psi_1)=\hat{f}(\tilde{\mathbf{x}})+\beta_0 s(\tilde{\mathbf{x}}),
\]
for fixed constants $\alpha_0<\beta_0$ independent of $\psi_1$, the truncated Gaussian predictive variance satisfies
\[
V^\prime_{\mathrm{base}}(\psi_1;\lambda)
=
c_{\alpha_0,\beta_0} V_{\mathrm{base}}(\psi_1;\lambda)
\]
for some constant $c_{\alpha_0,\beta_0}>0$ depending only on $\alpha_0,\beta_0$.
Moreover,
\[
\frac{(b-a)^2}{12}
=
\frac{(\beta_0-\alpha_0)^2}{12} V_{\mathrm{base}}(\psi_1;\lambda).
\]
Hence
\[
g_+(\psi_1)
=
\left((1-\eta)c_{\alpha_0,\beta_0}
+
\eta\frac{(\beta_0-\alpha_0)^2}{12}\right)
V_{\mathrm{base}}(\psi_1;\lambda),
\]
which is a positive multiple of $V_{\mathrm{base}}(\psi_1;\lambda)$.
Applying Lemma~\ref{lem:affine-invariance} yields that $g_+$ and
$V_{\mathrm{base}}$ have the same maximizer(s).

Therefore the bounding functions $g_-$ and $g_+$ attain their maxima
at $\psi_1^\star$, and form an envelope around the same double-descent peak.
\end{proof}

\end{document}